\documentclass[accepted]{uai2023} 

\usepackage[british]{babel}

\usepackage{natbib} 
\usepackage{mathtools} 
\usepackage{booktabs} 
\usepackage{tikz} 

\usepackage{amsmath}
\usepackage{amssymb}
\usepackage{mathtools}
\usepackage{amsthm}

\theoremstyle{plain}
\newtheorem{theorem}{Theorem}[section]

\newtheorem{lemma}[theorem]{Lemma}
\newtheorem{corollary}[theorem]{Corollary}
\theoremstyle{definition}
\newtheorem{definition}[theorem]{Definition}
\newtheorem{assumption}[theorem]{Assumption}
\theoremstyle{remark}

\usepackage{algorithm}
\usepackage{algorithmic}

\usepackage{enumitem}

\newcommand{\myeq}[1]{\ensuremath{\stackrel{\textit{\text{#1}}}{=}}}
\newcommand{\myleq}[1]{\ensuremath{\stackrel{\textit{\text{#1}}}{\leq}}}

\newcommand{\mypropto}[1]{\ensuremath{\stackrel{\textit{\text{#1}}}{\propto}}}
\newcommand{\algname}{FH-SMDP-UCRL}

\DeclareMathOperator*{\E}{\mathbb{E}}
\DeclareMathOperator*{\Var}{\mathbb{V}\mathrm{ar}}
\newcommand{\btau}{\bar{\tau}}

\usepackage{thmtools} 
\usepackage{thm-restate}
\usepackage{bbm}

\title{An Option-Dependent Analysis of Regret Minimization Algorithms \\
in Finite-Horizon Semi-Markov Decision Processes}

\author[1]{Gianluca Drappo}
\author[1]{Alberto Maria Metelli}
\author[1]{Marcello Restelli}
\affil[1]{%
    DEIB, Politecnico di Milano, Milan, 20133, Italy.
}
  
\begin{document}
\maketitle

\begin{abstract}
A large variety of real-world Reinforcement Learning (RL) tasks is characterized by a complex and heterogeneous structure that makes end-to-end (or flat) approaches hardly applicable or even infeasible. Hierarchical Reinforcement Learning (HRL) provides general solutions to address these problems thanks to a convenient multi-level decomposition of the tasks, making their solution accessible. Although often used in practice, few works provide theoretical guarantees to justify this outcome effectively. Thus, it is not yet clear when to prefer such approaches compared to standard flat ones. In this work, we provide an option-dependent upper bound to the regret suffered by regret minimization algorithms in finite-horizon problems. We illustrate that the performance improvement derives from the planning horizon reduction induced by the temporal abstraction enforced by the hierarchical structure. Then, focusing on a sub-setting of HRL approaches, the options framework, we highlight how the average duration of the available options affects the planning horizon and, consequently, the regret itself. Finally, we relax the assumption of having pre-trained options to show how in particular situations, learning hierarchically from scratch could be preferable to using a standard approach.
\end{abstract}

\section{Introduction}
\label{sec:introduction}
Hierarchical Reinforcement Learning \citep[HRL,][]{pateria2021hierarchical} is a learning paradigm that decomposes a long-horizon Reinforcement Learning~\citep[RL,][]{sutton2018reinforcement} task into a sequence of potentially shorter and simpler sub-tasks. The sub-tasks themselves could be further divided, generating a hierarchical structure organized in an arbitrary number of levels. Each of these defines a different problem, where the original action space is replaced by the set of sub-tasks available on the lower level, and the same could be replicated for multiple levels. Although, the actual state transition is induced only once the control reaches the leaf nodes, where the policies choose among the primitive actions (i.e., actions of the original MDP on top of which the hierarchy is constructed). For the higher levels, once a sub-task is selected, the control passes to the relative internal policy until its termination. This introduces the concept of \emph{temporal abstraction}~\citep{precup1997multi}, for what concerns the high-level policy, the action persists for a certain time, resulting in an actual reduction of the original planning horizon.

Several algorithms demonstrate outstanding performance compared to standard RL approaches in several long-horizon problems \citep{levy2019learning, vezhnevets2017feudal, bacon2017option, nachum2018data}. However, such evidence is mainly emerging in practical applications, and the theoretical understanding of the inherent reasons for these advantages is still underdeveloped. Only a few papers tried to justify these advantages theoretically, focusing on different aspects. For instance, \citet{mann2015approximate} studies the convergence of an algorithm that uses temporally extended actions instead of primitive ones. \citet{fruit2017regret} and the extension \cite{fruit2017exploration} focus on the exploration benefit of using options in average reward problems. More recently, \citet{wen2020efficiency} show how the MDP structure affects regret. In this paper, we seek to further bridge this theory-practice gap, following the intuition that a hierarchical structure in a finite-horizon problem would positively affect the sample complexity by reducing the planning horizon. This could help to discriminate among situations in which a hierarchical approach could be more effective than a standard one for this particular family of problems.  

\textbf{Contributions}~~The contributions of the paper can be summarized as follows. (1) We propose a new algorithm for the finite-horizon setting that exploits a set of \emph{fixed} options \citep{sutton1999between} to solve the HRL problem. (2) We conducted a regret analysis of this algorithm, providing an option-dependent upper bound, which, to the best of our knowledge, is the first in the Finite Horizon setting. This result could be used to define new objective functions for options discovery methods that would search for options that minimize this regret. (3) For the sake of our analysis, we formulate the notion of Finite Horizon SMDP and a \emph{performance difference lemma} for this setting. (4) Lastly, we provide an algorithm to relax the assumption of having options with fixed policies, and we demonstrate that there are situations in which such an approach could provide better guarantees in terms of sample complexity.

\textbf{Outline}~~In the following sections, we first introduce the problem and the notion used. Then, Section \ref{sec:motivation} describes the main motivation behind this work and the focus on the finite-horizon settings, and Section \ref{sec:fhsmdp} describes the new formalism introduced. The algorithm and its extension are described in Section \ref{sec:alg}, and the main result and its derivation are discussed in Sections \ref{sec:result} and \ref{sec:proof}. Finally, we describe in detail the related works (Section \ref{sec:relworks}) and discuss some further directions of research beyond the present work.

\section{Preliminaries}
In this section, we provide the necessary background that will be employed in the remainder of the paper.

\label{sec:preliminaries}
\paragraph{Finite-Horizon MDPs}~~A Finite Horizon Markov Decision Process \citep[MDP,][]{puterman2014markov} is defined as a tuple $\mathcal{M} = (\mathcal{S, A}, p, r, H)$, where $\mathcal{S}$ and $\mathcal{A}$ are the finite state and the primitive action spaces, respectively, $p(s'|s,a,h)$ is the transition probability function defining the probability of transitioning to state $s' \in \mathcal{S}$ by taking action $a\in \mathcal{A}$ in state $s\in\mathcal{S}$ at stage $h \in [H]$. $r(s,a,h)$ is the reward function that evaluates the quality of action $a\in \mathcal{A}$ when taken in state $s\in\mathcal{S}$ at stage $h\in [H]$, and $H$ is the horizon, which defines the duration of each episode of interaction with the environment. 
The behavior of an agent is modeled by a deterministic policy $\pi: \mathcal{S} \times [H] \rightarrow \mathcal{A}$ that maps states $s \in \mathcal{S}$ and stages $h \in [H]$ to actions. 

\paragraph{Semi-MDP}~~A Semi-Markov Decision Process \citep[SMDP,][]{baykal2010semi, cinlar2013introduction} is a generalization of the MDP formalism. It admits \emph{temporally extended} actions, which, contrary to \emph{primitive} ones (i.e., actions that execute for a single time step), can execute for a certain time during which the agent has no control over the decision process. A usual notion when treating SMDP is the \emph{duration} or \emph{holding time}, $\tau(s,a,h)$, which is the number of primitive steps taken inside a temporally extended action.

HRL builds upon the theory of Semi-MDPs, characterizing the concept of temporally extended action with basically two main formalisms \citep{pateria2021hierarchical}: sub-tasks \citep{dietterich2000hierarchical} and options \citep{sutton1999between}. For the sake of this paper, we focus on the options framework. 

\textbf{Options}~~An option \citep{sutton1999between} is a possible formalization of a temporally extended action. It is characterized by three components $o = (\mathcal{I}_o, \beta_o, \pi_o)$. $\mathcal{I}_o \subseteq \mathcal{S} \times [H]$ is the subset of states and stages pairs $(s,h)$ in which the option can start, $\beta_o: \mathcal{S} \times [H] \to [0,1]$ defines the probability that an option terminates in a specific state-stage pair, $\pi_o: \mathcal{S} \times [H] \to \mathcal{A}$ is the policy executed until its termination.  
An example of an option could be a pre-trained policy to execute a specific task in a control problem, such as picking up an object.

Exactly as stated by \citet[][Theorem 1]{sutton1999between} \textit{an MDP in which primitive actions $\mathcal{A}$ are replaced by options $\mathcal{O}$, becomes an SMDP.}
In this paper, we consider a hierarchical approach working for a two-level hierarchy. On the top, the goal is to find the optimal policy $\mu: \mathcal{S} \to \mathcal{O}$, which determines the optimal option for each state-instant pair. Once an option is selected, out of the SMDP's scope, its policy is executed until its termination, and the control returns to the high-level policy. An assumption is needed on the set of the given options \citep{fruit2017regret}.
\begin{assumption}[Admissible options]
The set of options $\mathcal{O}$ is assumed admissible, i.e., all options terminate in finite time with probability 1 and, $\{\forall o \in \mathcal{O},\, s \in \mathcal{S},\, \text{and } h \in [H], \ \exists  o' \in \mathcal{O}: \beta_o(s,h) > 0 \ \text{and} \ (s,h) \in \mathcal{I}_{o'}\}$. 
\end{assumption}
Lastly, we introduce an essential quantity for our analysis.

\textbf{Regret}~~The \emph{regret} is a performance metric for algorithms frequently used in provably efficient RL. For any starting state $s \in \mathcal{S}$, and up to the episode $K$, it is defined as: 
\begin{align}
\label{eq:regret}
    \textit{Regret(K)} \myeq{def} \sum_{k=1}^K V^*(s,1) - V^{\mu_k}(s,1)
\end{align}
and evaluates the performance of the policy learned until episode $k$, $V^{\mu_k}$ compared to the value of an optimal policy $V^*$.

\subsection{Notation}
In the following, we will use $\Tilde{O}(\cdot)$ to indicate quantities that depend on $(\cdot)$ up to logarithmic terms. 
$\mathbbm{1}(x=a)$ defines the indicator function
\begin{align*}
    \mathbbm{1}(x=a) \myeq{def} 
    \begin{cases}
        0, & \text{if}\ x \neq a \\
        1, & \text{if}\ x = a
    \end{cases}
\end{align*}
In the analysis, we denote optimistic terms with $\sim$ and empirical ones with $\land$, e.g., \textit{$\Tilde{p}$ and $\hat{r}$ are, respectively, the optimistic transition model and the estimated reward function.}

\section{Motivation and Intuition}
\label{sec:motivation}
Usually, in Reinforcement Learning, the complexity of a problem is highly correlated to the planning horizon, which is even more natural in finite-horizon MDPs. The regret analysis in the literature provides results for both the lower and upper bound on the regret paid by an algorithm in this setting, where there is a dependency on $H$.

Here comes our intuition, by using a hierarchical approach, we can intrinsically reduce the planning horizon because the number of decisions taken in $H$ time steps is scaled by a term closely related to the average duration of each action, and thus, also the complexity scales with this quantity. 
In addition, if the sub-tasks themselves need to be learned because policies are not provided, simplification can also be induced in these new problems. Under certain assumptions, they could have shorter horizons, and an agent during the training can focus just on smaller regions of the entire state space. Furthermore, the learning could be further guided by an additional reward that could better specify the singular sub-problem.\\
Starting from this intuition, we analyze the performance of an algorithm in a Finite Horizon Semi-MDP, considering a set of pre-trained options and, afterward, its extension, which incorporates a first phase of options learning. Lastly, we provide a comparative study with its flat counterpart.

\section{Finite Horizon SMDP}
\label{sec:fhsmdp} 
In this section, we present a new formalism, \emph{Finite-Horizon Semi-Markov Decision Processes}, that combines the notion used in FH-MDP with the concept of temporal abstraction. 

A finite-horizon semi-MDP is defined as a tuple $\mathcal{SM} = (\mathcal{S, O}, p, r, H)$,  where $\mathcal{S}$ and $\mathcal{O}$ are the finite state and the temporally extended action spaces, respectively, $p(s',h'|s,o,h)$ is the probability of ending to state $s'\in \mathcal{S}$, after $(h'-h)$ steps, by playing the temporally extended action $o \in \mathcal{O}$ in state $s\in \mathcal{S}$ at stage $h \in [H]$. On the other hand, $r(s,o,h)$ is the expected reward accumulated until the termination of the temporally extended action $o \in \mathcal{O}$ played in state $s \in \mathcal{S}$ at  stage $h \in [H]$ of the episode. Finally, $H$ is the horizon of interaction, and still $\tau(s,o,h)$ is the number of primitive steps taken inside the temporally extended action. The agent's behavior is modeled by a deterministic policy $\mu :
 \mathcal{S} \times [H] \rightarrow \mathcal{O}$ mapping a state $s\in \mathcal{S}$ and a stage $h \in [H]$ to a temporally extended action. The goal of the agent is to find a policy $\mu^*$ that maximizes the value function, defined as the expected sum of the rewards collected over the horizon of interaction and recursively defined as:
\begin{align*}\resizebox{.48\textwidth}{!}{$\displaystyle
    V^{\mu}(s, h) = \E_{(s',h')\sim p(\cdot |s,\mu(s,h),h)}\Big[r(s,\mu(s,h),h) + V^{\mu}(s', h') \Bigr],$}
\end{align*}
with the convention that $V^{\mu}(s,H)=0$. The value function of any optimal policy is denoted by $V^*(s,h) \coloneqq V^{\mu^*}(s,h)$.

\section{Algorithms}
\label{sec:alg}
\algname is a variant of the algorithm presented in \cite{fruit2017exploration}, which in turn is inspired by UCRL2 \citep{auer2008near} and adapted for FH-SMDPs. This family of algorithms implements the principle of "\textit{optimism in the face of uncertainty}", which states that when interacting in unknown environments---with an unknown model---the decisions have to be guided by a trade-off between what we believe is the best option and by a term representing the uncertainty on our estimates. More formally, it is introduced the so-called \emph{exploration bonus}, which quantifies the uncertainty level on our estimations of the model, computed from the observed samples. This exploration bonus is used to regulate the \textit{exploration-exploitation} dilemma, inducing the algorithm to explore regions of the space with high uncertainty instead of sticking to what seems to be the optimal solution and to overcome situations in which the optimal solution resides in a region not yet discovered.

However, a direct application of these algorithms in our setting is unfeasible, as they are designed for infinite-horizon average reward settings.
Due to the lack of methods operating in these settings, we need to design a new algorithm for finite-horizon SMDPs following the same paradigm.

As displayed by Algorithm \ref{alg:given_opt}, at each episode $k$, we compute an estimate of the SMDP model, by computing, from the collected samples up to episode $k$, the empirical transition probability $\hat{p}_k$ and the reward function $\hat{r}_k$.
\begin{align}\label{eq:est-p}\resizebox{.43\textwidth}{!}{$\displaystyle
    \hat{p}_{k}(s',h'|s,o,h)= \frac{\sum_{i=1}^{k-1} \mathbbm{1}((s, o, s', h, h')_i = (s,o,h,s',h'))}{n_{k}(s,o,h)}$}
\end{align}
\begin{align}\label{eq:est-r}\resizebox{.25\textwidth}{!}{$\displaystyle
    \hat{r}_{k}(s,o,h)= \frac{\sum_{i=1}^{k-1} r_{i}(s, o, h)}{n_{k}(s,o,h)}$}
\end{align}
We then redefine the confidence intervals of these two quantities, $\beta^p_k$ and $\beta^r_k$, respectively as
\begin{align}
    \label{eq:ci-r}
    &\beta^r_{k}(s,o,h)\ \propto\ \sqrt{\frac{2 \hat{\Var}(r) \ln 2/\delta}{n_k(s,o,h)}} + \frac{7\ln2/\delta}{3(n_k(s,o,h-1)},\\
    \label{eq:ci-p}
    &\beta^p_{k}(s,o,h)\ \propto\ \sqrt{\frac{S\log\bigl(\frac{n_{k}(s,o,h)}{\delta}\bigl)}{n_{k}(s,o,h)}},
\end{align}
where $\Hat{\Var}(r)$ is the sample variance of the reward function.
From the estimates and the confidence intervals just defined, we can build the confidence sets $B_k^p$ and $B_k^r$, which contain, with high probability, the true model. Being $\mathcal{SM}_k$ the set of plausible SMDPs, characterized by rewards and transition within the confidence sets, with $\mathcal{SM}_k$ and an adaptation of \textit{extended value iteration} \citep{auer2008near}, for FH-SMDP (Algorithm \ref{alg:evi}), we can compute the optimistic policy $\tilde{\mu_k}$ and the relative optimistic value function $\Tilde{V}^{\mu_k}$.
Then, by playing this policy for an entire episode, we collect new samples and restart the process for the next episode $k+1$. 

\subsection{Option Learning}
By relaxing the assumption of having a set of pre-trained options, considering known just their initial state set and termination conditions, we can enhance characteristics of problems more suited to be solved with a hierarchical approach even when no pre-trained policies are provided.

We present a model-based algorithm divided into two phases, which initially learns each option policy individually, and then exploits them to solve the SMDP with \algname. As Algorithm \ref{alg:opt-learn} shows, each option is considered as a single FH-MDPs, defined based on its initial-state set and termination probability, as $\mathcal{M}_{o} = (\mathcal{S}_o, \mathcal{A}_o, p, r_o, H_o)$ where $S_o \subseteq S$, $A_o \subseteq A$, $H_o \leq H$, which means that each option operates on a restricted portion of the original problem, for a certain fixed horizon $H_o$. The option's optimal policy is the policy of the relative sub-FH-MDP computed until episode $K_o$, which is the number of episodes assigned to each option.

Nevertheless, if no assumption on the reward function is defined, the options' optimal policies could be sub-optimal regarding the optimal policy computed by a standard approach for that portion of the MDP, being the option's scope limited to a certain part of the MDP with the impossibility of having feedback on what happens after its termination.\\
Therefore we need to state:
\begin{assumption}
Given an MDP $\mathcal{M} = (\mathcal{S, A}, p, r, H)$ and a set of options $o \in \mathcal{O}$. Define $\pi_o^*$ as the optimal policy of the option $o$ learned individually on the sub-MDP $\mathcal{M}_{o} = (\mathcal{S}_o, \mathcal{A}_o, p, r_o, H_o)$ with $S_o \subseteq S$, $A_o \subseteq A$, and $H_o \leq H$. The reward function $r_o$ of the sub-MDP $\mathcal{M}_o$, which could differ from $r$, ensure that 
\begin{align*}
    \pi^*(s) = \pi_o^*(s) ~~ \forall s \in S_o 
\end{align*}
with $\pi^*(s)$ the optimal policy on $\mathcal{M}$.
\end{assumption}
This assumption guarantees that the computed option's optimal policy equals the optimal policy of the entire problem in the option's region.

\begin{algorithm}[t]
\caption{UCRL-FH-SMDP} \label{alg:given_opt}
\begin{algorithmic}[1]
    \REQUIRE{$\mathcal{S, O} \text{ with fixed policies, } H$\\
    Initialize $\mu_0$ at random and $Q_{1}(s,o,h) = 0$ for all $(s,o,h) \in \mathcal{S} \times \mathcal{O} \times [H]$}
    \STATE{Execute $\mu_0$ for $H$ steps and collect tuples $(s,o,h,s',h')$ and $r(s,o,h)$ to store in $\mathcal{D}_1$}
    \FOR{$k = 1,\dots,K$}
        \STATE{Compute $n_k(s,o,h)$}
        \STATE{Estimate empirical SMDP $\widehat{\mathcal{SM}}_k = (\mathcal{S, O},  \hat{p}_{k}, \hat{r}_{k})$ with equations \ref{eq:est-p}, \ref{eq:est-r}.}
        \STATE{Compute the confidence sets $B^r_{k} (s,o,h)$ and $B^p_{k} (s,o,h)$} using the confidence interval (Eq. \ref{eq:ci-r}, \ref{eq:ci-p})
        \STATE{Planning with Backward Induction for $\mu_{k}$, using an adaptation to finite horizon of \textit{Extended Value Iteration} \citep{auer2008near}} (Algorithm \ref{alg:evi})
        \FOR{$h = 1,\dots,H$}
            \STATE{Execute $o = \mu_{k}(s, h)$ until it terminates}
            \STATE{Observe $(s', h')$ and $r(s, o, h)$}
            \STATE{Add the tuple $(s,o,h,s',h')_k$ and $r_k(s,o,h)$ to $\mathcal{D}_{k+1}$}
            \STATE{Set $h = h'$}
        \ENDFOR
    \ENDFOR
\end{algorithmic}
\end{algorithm}

\begin{algorithm}[t]
\caption{Extended Value Iteration for FH-SMDP} \label{alg:evi}
\begin{algorithmic}[1]
    \STATE{\textbf{Input:} $\mathcal{S,O}, B^r_k, B^p_k$}
    \STATE{Set $Q_{H+1}(s,o) = 0$ for all $(s,o) \in \mathcal{S \times O}$}
    \FOR{$h = H \dots 1$}
        \FOR{$(s,o) \in \mathcal{S \times O}$}
            \FOR{$h' = h+1 \dots H+1$}
                \STATE{Compute}
            \begin{align*}
            \resizebox{.4\textwidth}{!}{$\displaystyle
                    Q_{hk}(s,o) = \max_{r \in B_k^r(s,o,h)} r(s,o,h) + \max_{p \in B_k^p(s,o,h)} \E_{s',h' \sim p(\cdot,\cdot|s,o,h)}[V_{h',k}(s')]$}\\
            \resizebox{.32\textwidth}{!}{$\displaystyle
                    V_{h'k}(s) = \min\{H-(h'-1), \max_{o \in \mathcal{O}}Q_{h'k}(s,o)\}$}
                \end{align*}
            \ENDFOR
        \ENDFOR
    \ENDFOR
    \STATE{\textbf{Output:} $\mu_k(s,h) = \operatorname{argmax}_{o \in \mathcal{O}}Q_{hk}(s,o)$}
\end{algorithmic}
\end{algorithm}

\begin{algorithm}
\caption{Option Learning} \label{alg:opt-learn}
\begin{algorithmic}[1]
    \STATE{\textbf{Input:} $\mathcal{S}_o,\mathcal{A}_o, H_o, K_o$, and $B^r_k(s,a), B^p_k(s,a)$ which are respectively the confidence sets of the flat model.}
    \STATE{Set $Q_{H_o+1}(s,a) = 0$ for all $(s,a) \in \mathcal{S_o \times A_o}$}
    \FOR{$k = 1, \dots, K_o$}
        \STATE{Compute $n_k(s,a,h)$}
        \STATE{Estimate empirical MDP $\widehat{\mathcal{M}}_k = (\mathcal{S}_o, \mathcal{A}_o, \hat{p}_{k}, \hat{r}_{k})$, with an adaptation of eq. \ref{eq:est-p} and \ref{eq:est-r} for the flat model.}
        \STATE{Planning by backward induction for $\pi_{o_{hk}}$ with Extended Value Iteration for FH-MDPs \citep{glp2020aaaitutorial}, in the horizon $H_o$.}
        \STATE{Play $\pi_{o_k}$ for an episode to collects new samples.}
    \ENDFOR
    \STATE{\textbf{Output:} $\pi^{K_o}_o$.}
\end{algorithmic}
\end{algorithm}

\section{Main Results}
\label{sec:result} 
In this section, we present the main contributions of the paper, which in particular are an upper bound on the regret of \algname\ that highlights particular problem-dependent features and, an upper bound on the regret of its extension including a first phase of options learning. 
\begin{restatable}{theorem}{regretsmdp}\label{thm:regret}
    Considering a non-stationary Finite Horizon SMDP $\mathcal{SM}$ and a set of options $\mathcal{O}$, with bounded primitive reward $r(s,a) \in [0,1]$. The regret suffered by algorithm \algname, in $K$ episodes of horizon $H$ is bounded as:
    \begin{align*}
    \textit{Regret(K)} \leq \tilde{O}\Bigg(\Big(\sqrt{SOKd^2}\Big) \bigg(\overline{T} + \sqrt{S} H\bigg)\Bigg)
    \end{align*}
    with probability $1-\delta$.\\
    Where:
    \begin{align*} 
    \label{eq:T_bar}
    \overline{T} &= \max_{s,o,h} \sqrt{\E[\tau(o, s, h)^2]}\\ &= \max_{s,o,h} \sqrt{\E[\tau(o, s, h)]^2 + \operatorname{Var}[\tau(s,o,h)]},
    \end{align*}
    $\tau$ is the holding time, and $d$ describes the expected number of decisions taken in one episode that is $d \approx H/\bar{\tau}$, with $\bar{\tau}$ the average duration of the set of options.  
\end{restatable}

This result introduces one of the main contributions, an option-dependent upper bound on the regret in FH-MDP with options, not worst-case as in \cite{fruit2017exploration}. A dependency on the properties of an option set is introduced, embodied into both $\overline{T}$ and $d$. The former gives the same interesting consideration already underlined by \citet{fruit2017exploration}, whereby the extra cost of having actions with random duration is only partially additive rather than multiplicative.
On the other hand, the latter emphasizes the real benefit of using a hierarchical approach over a flat one. The longer the expected duration of the set of options, the more the effective horizon of the SMDP, $d$, decreases. Indeed, $d \approx \frac{H}{\bar{\tau}}$, with $\Bar{\tau}$ the average holding time of the options.
Notice that there is a $\sqrt{d}$ worsening factor, which comes from the fact that we consider a non-stationary MDP in the analysis. This outcome is common in finite-horizon literature \citep{azar2017minimax, dann2017unifying, zanette2018}, where, instead, the regret increases by a factor of $\sqrt{H}$.

Let's now analyze the regret suffered by the two-phase algorithm that first learns each option policy and then subsequently solves the relative SMDP.
\begin{restatable}{theorem}{twophase} \label{thm:two-phase}
The regret paid by the two-phase learning algorithm until the episode K is:
\begin{align*}
    Regret(K) \leq \Tilde{O}\Big(K^{\frac{2}{3}}\sqrt[3]{H^5_oS^2_oA_oO} + \frac{H^2S}{H_o}\sqrt{OK}\Big)
\end{align*}
with $H_o$ the fixed horizon of each option $o \in \mathcal{O}$, $S_o$, and $A_o$ the upper bounds on the cardinality of the state and action space of the sub-FH-MDPs.
\end{restatable}
As mentioned before, we consider a situation where we allocate $K_o$ episodes for each option learning, hence $T_1 = \sum_{o \in O} K_o H_o$, and $K_2$ episodes for the SMDP policy learning at fixed options policies. Being the algorithm divided into two phases, we expected a regret that scales with $K^{2/3}$, a study of a more efficient algorithm is left for future works.

We now would like to understand if there are any situations in which such an approach could produce more benefit compared to a standard one, to examine if there are any classes of problems in which learning using a hierarchical approach almost from scratch should be preferred to a standard one. Therefore we conduct a comparison of regrets paid by this algorithm and a standard flat-one, UCRL2 in particular.

We recall that the regret of UCRL2 adapted for non-stationary FH-MDPs \citep{glp2020aaaitutorial} is $Regret(UCRL2-CH) \leq \tilde{O}(H^2S\sqrt{AK})$, thus the ratio between the two regret bounds, $\mathcal{R}$, is
\begin{align*}
    \mathcal{R} =\frac{Regret_{SMDP}}{Regret_{MDP}} \leq \frac{K^{2/3}(H^5_o S^2_o A_o O)^{1/3}}{H^2 S\sqrt{AK}}
\end{align*}
By considering particular relations between the option-MDP and the original one, where $A_o = \alpha A$, $S_o = \alpha S$ and $H_o = \alpha H$, we can rewrite this ratio as:
\begin{align}
    \mathcal{R} \leq \frac{K^{1/6}\alpha^{3/8}O^{1/3}}{(HS)^{1/3}A^{1/6}}
\end{align}
When $\mathcal{R} \leq 1$ is clearly beneficial to use this approach instead of a standard one. Thus, by imposing this assumption, we can find the maximum number of episodes for which this constraint is satisfied. 
\begin{align*}
    K \leq \frac{H^2S^2A}{\alpha^{16} O^2}
\end{align*}
\textbf{Remark}~~While in general, a comparison of upper bounds is potentially loose, in this case, the two upper bounds have been derived with similar techniques; hence they would be \emph{similarly loose}.

\paragraph{Final Sub-Optimality Remark.} In both results, we need to consider that, by using a defined option set, we are introducing a bias. It could be that the optimal policy on the flat problem is irreproducible by a concatenation of the policies of the options chosen by the optimal high-level policy $\mu^*$. This is because the structure introduced by the options can cause some states to become inaccessible for the high-level SMDP. This issue is also treated by \citet{fruit2017exploration} and produces an additional term on the regret equal to $V^*(M) - V^*(M_O)$, where $M$ is the primitive MDP, and $M_O$ the same MDP with options.

\subsection{Renewal Process}
The expected number of options played in one episode $d$, clearly depends on the random duration of each of these options; hence it is itself a random variable, and we would like to bound it with some quantity. Resorting to the \textit{Renewal Theory} \citep{smith1958renewal}, this corresponds to the \textit{Renewal Function} $m(t)$.
\begin{definition}[Renewal Process]
Let $S_1, S_2 \dots $ be a sequence of i.i.d. random variables with finite and non-zero mean, representing the random time elapsed between two consecutive events, defined as the holding time. 
For each $n > 0$ we define $J_n = \sum_{i=1}^n S_i,$
as the time at which the $n^{th}$ event of the sequence terminates.
Then, the sequence of random variables $X_t$, characterized as
\begin{align}\label{eq:renew_proc}
    X_t = \sum_n^\infty \mathbb{I}_{\{J_n \leq t\}} = \operatorname{sup} \{n : J_n \leq t \}
\end{align}
constitutes a Renewal Process $(X_t)_{t\geq0}$, representing the number of consecutive events that occurred up to time $t$.
\end{definition}
\begin{definition}[Renewal Function]
Considering a renewal process $(X_t)_{t\geq0}$, the renewal function $m(t)$ is the expected number of consecutive events that occurred by time $t$.
\begin{align*}
        m(t) = \E[X_t]
\end{align*}
\end{definition}

Hence, it is possible to take inspiration from a bound of the renewal function to bound the expected number of options played in one episode.
\begin{restatable}{lemma}{optioncount}\label{lemma:opt-count}[Bound on number of options played in one episode]
    Considering a Finite Horizon SMDP $\mathcal{SM}$ with horizon H and, $O$ options with duration $\tau_{min} \leq \tau \leq \tau_{max}$ and $\min_o(\E[\tau_o])$ the expected duration of the shorter option.
    The expected number of options played in one episode $d$ can be seen as the renewal function $m(t)$ of a renewal process up to the instant $H$. With probability $1-\delta$ this quantity is bounded by
    \begin{align*}
        d < \sqrt{\frac{32(\tau_{max}- \tau_{min})H(\ln2 - \ln{\delta})}{(\min_o(\E[\tau_o]))^3}} + \frac{H}{\min_o(\E[\tau_o])}
    \end{align*}
\end{restatable}
Refer to the appendix for detailed proof of this result.

\subsection{Fixed option length}
Let's now analyze a particular case to clarify the claim introduced with Theorem \ref{thm:regret}.
A scenario in which the given options are deterministic with fixed length.
\begin{corollary}
Considering a non-stationary Finite Horizon SMDP $\mathcal{SM}$ and a set of deterministic option $O$ with fixed duration $\Bar{\tau}$, the regret payed by \algname, after K steps is upper bounded by:
\begin{align*}
    \textit{Regret(K)} \leq \tilde{O}\Bigg(\frac{H}{\Bar{\tau}}\Big(\sqrt{SOK}\Big) \bigg(\Bar{\tau} + \sqrt{S} H\bigg)\Bigg)
\end{align*}
\end{corollary}
\textit{Proof.} The result is trivially derived by substituting $d$ with the actual number of decisions taken in one episode, which now is a defined number equal to $H/\Bar{\tau}$. Then, the same applied for $\overline{T}$ that, considering options of length $\bar{\tau}$, is exactly $\Bar{\tau}$.

This bound clearly shows a dependency on the choice of the options set. The second term, which is the dominant one, is mitigated by the $\sqrt{\Bar{\tau}}$,  thus reducing the sample complexity as expected.
The other $\sqrt{\frac{H}{\Bar{\tau}}}$, as for Theorem \ref{thm:regret}, comes from the non-stationarity of the SMDP.

\subsection{Derivation of FH-MDP and parallelism with AR-SMDP.}
To further strengthen the obtained result, we can show that considering some assumptions, we can derive the upper bound by \citet{auer2008near} adapted to FH-MDPs \citep{glp2020aaaitutorial}.
\paragraph{Finite Horizon MDP.}
Referring to the result provided by \citet{auer2008near} adapted to the finite horizon case \citep{glp2020aaaitutorial}, the regret in Finite Horizon MDPs scales with $\Tilde{O}(HS\sqrt{AT})$, or with $\Tilde{O}(H^{\frac{3}{2}}S\sqrt{AT})$ when the MDP in non-stationary. If, in our upper bound, we substitute the cardinality of the option set $O$, with the primitive-action space $A$. This leads to having $\overline{T} = 1$ and $d = H$ because the primitive actions, by definition, terminate after a single time step. Thus, the average duration of these single-step options is 1, and the number of decisions taken in one episode is exactly $H$. Then, having bounded primitive reward $r(s,a) \in [0,1]$, we can write our result as $\Tilde{O}(H^{\frac{3}{2}}S\sqrt{AKH})$ and considering the definition of $T = KH$ \citep{dann2017unifying, azar2017minimax, zanette2018}, we obtain the same equation.

\textbf{Remark}~~ We are aware of the tighter upper bounds in the Finite Horizon literature by \cite{azar2017minimax}, that get rid of a $\sqrt{HS}$, by tightly bounding the estimation error $(\tilde{p} - p)\tilde{V}^{\mu_k}$ and their exploration bonus in terms of the variance of $V^*$ at the next state, and by using empirical Bernstein and Freedman's inequalities \citep{maurer2009empirical, freedman1975tail}. 
However, with this work, our main focus is to emphasize the role played by the options set's composition instead of providing a very tight analysis. We still think the same tricks could be used in our analysis to tighten the bound, but we leave that for future work.
\paragraph{Parallelism with Average Reward Setting.} \citet{fruit2017exploration} showed that the regret in SMDPs with options when considering bounded holding times, and $R_{max} = 1$ scales with $\Tilde{O}(D_{\mathcal{O}}S_{\mathcal{O}}\sqrt{On} + T_{max}\sqrt{S_{\mathcal{O}}On})$. On the other hand, considering the same assumptions, our result becomes of order $\Tilde{O}(HS\sqrt{OKd^2}+ T_{max}\sqrt{SOKd^2})$. 
It is clearly impossible to derive one setting from the other. Nevertheless, we can spot some similarities between the two bounds. We can say that for finite-horizon problems, the diameter $D$ coincides with the horizon $H$ \citep{glp2020aaaitutorial}. Besides, $Kd$ is exactly equal to $n$, the number of decisions made up to episode $K$. The state space $S$ is the state space of the SMDP in our formalism, which is the definition provided for $S_O$ in \citet{fruit2017exploration}. Consider, again, that the additional $\sqrt{d}$ comes from the fact that we refer to a non-stationary FH-SMDP.

Thus, we prove that our result is a generalization of the case of FH-MDP and closely relates to the result presented for the Average Reward Setting.

\section{Proofs sketch}
\label{sec:proof}
In this section, we provide the sketch proofs of theorem \ref{thm:regret} and theorem \ref{thm:two-phase}. Please refer to the appendix for all the details.

\subsection{Sketch proof of theorem \ref{thm:regret}}
We defined the regret in finite horizon problems as in eq. \ref{eq:regret}. Optimistic algorithms work by finding an optimistic estimation of the model of the environment to compute the optimistic value function and the optimistic policy. 
Considering how the confidence sets are constructed, we can state that $\Tilde{p} \geq p$ and $\Tilde{r} \geq r$, where terms without tilde are the real one, hence, $V^*(s,h) \leq \Tilde{V}^{\mu_k}(s, h)$ for all $h$. Thus, we can bound eq. \ref{eq:regret} with
\begin{align}\label{eq:regret-opt}
    \textit{Regret(K)} \myleq{opt} \sum_{k=1}^K \Tilde{V}^{\mu_k}(s,1) - V^{\mu_k}(s,1)
\end{align}

Let's now introduce a Performance Difference Lemma for FH-SMDPs. 

\begin{restatable}{lemma}{pdl}[Performance Difference Lemma for FH-SMDP]\label{lemma:pdl}
Given two FH-SMDPs $\hat{M}$ and $\tilde{M}$ with horizon $H$, and respectively rewards $\hat{r}$, $\Tilde{r}$ and transition probabilities $\hat{p}$, $\Tilde{p}$. The difference in the performance of a policy $\mu_k$ is:
\begin{align*}
    &\tilde{V}^{\mu_k}(s,1)-\hat{V}^{\mu_k}(s,1) \\ 
    &= \hat{\mathbb{E}}\bigg[ \sum_{i=1}^H \Big(\big(\tilde{r}(s_i,o_i,h_i)-\hat{r}(s_i,o_i,h_i)\big)\\
    &+ \big(\tilde{p}(s_{i+1}, h_{i+1}|s_i,o_i,h_i)-\hat{p}(s_{i+1}, h_{i+1}|s_i,o_i,h_i)\big)\\
    &\tilde{V}^{\mu_k}(s_{i+1},h_{i+1})\Big) \mathbbm{1}\{h_i < H\} \bigg] 
\end{align*}
where $\hat{\E}$ is the expectation taken w.r.t. $\hat{p}$ and $\mu_k$.
\end{restatable}
Note that the summation steps are not unitary but skip according to the length of the transitions $h'-h$.
The derivation of this lemma follows the one provided by  \citet{dann2017unifying} for FH-MDPS that is commonly used in literature \citep{azar2017minimax, zanette2018}. Check the appendix for further details.

Now we can use lemma \ref{lemma:pdl} to substitute the difference of value function in eq. \ref{eq:regret-opt} and we can upper bound both the difference of $r$ and $p$, with 2 times their confidence intervals and the optimistic value $\tilde{V}^{\mu_k}(s_{i+1},h_{i+1})$ with the horizon $H$ - we consider bounded primitive reward $r(s,a) \in [0,1]$.
\begin{align*}
    &\textit{Regret(K)} \leq \sum_{k=1}^K\mathbb{E}\bigg[ \sum_{i=1}^H \Big(2\beta_k^r + 2\beta_k^p H\Big)\mathbbm{1}\{h_i < H\} \bigg] 
\end{align*}
In the Finite-Horizon literature\citep{dann2017unifying, zanette2018}, two terms are commonly used in the proofs: (1) $w_k(s,o,h)$ that is the probability of taking the option $o$, in state $s$ at time step $h$, which clearly depends on the policy $\mu_k$ and the transition probability of the real SMDP, (2) $L_k$, which defines the set of episodes visited sufficiently often, and the set of $(s, o, h)$ that were not visited often enough to cause high regret.
Therefore, for using the same approach to conduct the proof, we can substitute the expectation $\E$, which is taken w.r.t. the policy $\mu_k$ and the real transition probability $p(s',h'|s,o,h)$, with
\begin{align*}
    \sum_{(s,o,h) \in L_k} w_k(s,o,h) 
\end{align*}
We defined the confidence intervals of $r$ and $p$, as in the equations \ref{eq:ci-r}, \ref{eq:ci-p}, respectively using Empirical Bernstein Inequality \citep{maurer2009empirical}, \citet{hoeffding1963probability} and \citet{weissman2003inequalities}.

By substituting these definitions and the term just introduced, we get, up to numerical constants, that the regret is bounded by
\begin{align*}\resizebox{.5\textwidth}{!}{$\displaystyle
  \sum_k \sum_{i \in [H]} \sum_{(s,o,h) \in L_k} \frac{w_{k}(s_i,o_i,h_i)}{\sqrt{n_k(s_i,o_i,h_i)}} \bigg(\sqrt{\Hat{\Var}(r)} +  \sqrt{S}H + \frac{1}{\sqrt{n_k(s,o,h)}}\bigg)$}
\end{align*}

\begin{restatable}{lemma}{zanettesmdp}
\label{lemma:zan_smdp}
Considering a non-stationary MDP M with a set of options as an SMDP $M_{\mathcal{O}}$ \citep{sutton1999between}. In $M_{\mathcal{O}}$ the number of decisions taken in the $k^{th}$-episode is a random variable $d$ and
\begin{align*}\resizebox{.48\textwidth}{!}{$\displaystyle
        \sum_{i \in H} \sum_{(s,o) \in L_k} w_k(s_i,o_i,h_i) \mathbbm{1}\{h_i < H\} = d \ \text{with} \ \{\forall k: d \leq H\}$}
\end{align*}
Therefore, the following holds true:
    \begin{align*}\resizebox{.48\textwidth}{!}{$\displaystyle
        \sum_k \sum_{i \in H} \sum_{(s,o) \in L_k} w_k(s_i,o_i,h_i) \sqrt{\frac{1}{n_k(s_i,o_i,h_i)}} = \tilde{O} \biggl(\sqrt{SOKd^2} \biggr)$}
    \end{align*}
or, using the same notation used in \cite{fruit2017exploration}, $\tilde{O}(\sqrt{SOKd^2})$, with $n = Kd$ the number of decisions taken up to episode $K$.
\end{restatable}
Substituting the result of Lemma \ref{lemma:zan_smdp} in the equation of the regret, we get
\begin{align*}\resizebox{.48\textwidth}{!}{$\displaystyle
    \textit{Regret(K)} \leq \tilde{O}\Bigg(\Big(\sqrt{dSOn}\Big) \bigg(\sqrt{\Hat{\Var}(r)} + \sqrt{S} H\bigg) + dSO\Bigg)$}
\end{align*}
where as mentioned above, $d$ is the expected number of decision steps taken in one episode, $n$ is the total number of decisions taken up to episode $k$, and $\Hat{\Var}(r)$ is the empirical variance of the reward that emerged from the use of Empirical Bernstein inequality.
A dependency on the variance of the reward is not that explainable for what we want to show; hence we upper bound this term by the square root of the empirical variance $\overline{T}$ of the duration of the options seen up to episode $k$, and this complete the proof.

\subsection{Sketch proof of theorem \ref{thm:two-phase}}
In order to prove the regret paid by the two-phase algorithm, we first consider that we can write the regret as the sum of the regret paid in the first phase and the regret paid in the second one, plus an additional bias term.
In the first phase, we pay full regret for each option learning, then the maximum average regret considering the option learning as a
finite horizon MDP with horizon $H_o$, and the regret of the SMDP learning with fixed options
\begin{align}\label{eq:twophase-ex}\resizebox{.42\textwidth}{!}{$\displaystyle
    Regret(K) \leq \sum_{o \in O}K_oH_o + K_2 \max_{o \in O} \frac{1}{K_o}H^2_oS_o\sqrt{A_oK_o}+ HS\sqrt{Od^2K_2}$}
\end{align}
Where $K_2$ are the episodes used for the SMDP learning, and $K = \sum_{o \in O} K_o + K_2$.
Then considering that we allocate $K_o$ episodes for each option learning, and $A_o$ and $S_o$ are respectively the upper bounds on the action-space cardinality and state-space cardinality of the options set, we can get rid of the $\max_{o \in \mathcal{O}}$. 
Now bounding $K_2 \leq K$, we can find the optimal $K_o$ in closed form, and substituting it in Equation \ref{eq:twophase-ex}, we conclude the proof.

\section{Related Works}
\label{sec:relworks}
In the FH-MDP literature, several works provide an analysis of the regret of different algorithms.
\citet{osband2016lower} present a lower bound that scales with $\Omega(\sqrt{HSAT})$. On the other hand, many other works propose various upper bounds for their algorithm. The most common upper bound is the adaptation of \cite{auer2008near} proposed by \citet{glp2020aaaitutorial}, which is of the order of $O(HS\sqrt{AT})$. This result has then been improved in the following papers. An example is \cite{azar2017minimax}, which proposes a method with an upper bound on the regret of $O(\sqrt{HSAT})$ that successfully matches the lower bound.
As mentioned above, both upper and lower bounds depend on $H$.

Nevertheless, few works focused on theoretically understanding the benefits of hierarchical reinforcement learning approaches, and, to the best of our knowledge, this is the first to analyze these aspects in FH-SMDPs.
To conduct our analysis, we take inspiration from the paper by \citet{fruit2017exploration}, in which they propose an adaptation of UCRL2 \citep{auer2008near} for SMDPs. They first study the regret of the algorithm for general SMDPs and then focus on the case of MDP with options, providing both a lower bound and a worst-case upper bound. This work was the first that theoretically compares the use of options instead of primitive actions to learn in SMDPs. Nonetheless, it focuses on the average reward setting to study how it is possible to induce a more efficient exploration by using options, and it assumes fixed options. Differently, we aim to analyze the advantages of using options to reduce the sample complexity of the problem, resorting to the intuition that temporally extended actions can intrinsically reduce the planning horizon in FH-SMDPs.
Furthermore, we provide an \textit{option-dependent} upper bound, instead of a worst-case one, that better quantifies the impact of the option duration on the regret. Other works providing a theoretical analysis of hierarchical reinforcement learning approaches are \cite{fruit2017regret}, which is an extension of the aforementioned work in which the need for prior knowledge of the distribution of cumulative reward and duration of each option is relaxed. Even in this case, they consider the average reward setting, and the objective is identical.

Then, \citet{mann2015approximate} study the convergence property of Fitted Value Iteration (FVI) using temporally extended actions, showing that a longer duration of options and pessimistic estimates of the value function lead to faster convergence.
Finally, \citet{wen2020efficiency} demonstrate how patterns and substructures in the MDP provide benefits in terms of planning speed and statistical efficiency. They present a Bayesian approach exploiting this information, and they analyze how sub-structure similarities and sub-problems' complexity contribute to the regret of their algorithm.

\section{Conclusions}
In conclusion, we propose a new algorithm for Finite Horizon Semi Markov decision processes called \algname, and we provide theoretical evidence that supports our original claim. Using hierarchical reinforcement learning, it is provably possible to reduce the problem complexity of a Finite Horizon problem when using a well-defined set of options. This analysis is the first for FH-SMDP and provides a form of option-dependent analysis for the regret that could be used to define objectives for options discovery methods better. Furthermore, by relaxing the assumption of having a set of fixed options' policies, we were able to provide insights on classes of problems in which a hierarchical approach from scratch would still be beneficial compared to a flat one. In the future, we would like to improve the algorithm proposed for options learning to tighten the theoretical guarantees and further characterize this family of problems. Finally, we would like to investigate, following the ideas of \citet{wen2020efficiency}, how the structure of the MDP could appear in our bound, which, in our opinion, is a fundamental point to put another brick in the direction of total understanding on the promising power of HRL.

\bibliography{biblio}

\newpage
\appendix
\onecolumn
\section{Detailed proof of Theorem \ref{thm:regret}}
\subsection{Notation and Setting}
First, we need to define the value function of an SMDP. In \cite{sutton1999between} it is defined as a formalism for MDP with options, that itself, by the demonstration presented in the same article, is an SMDP. \\
In our case, however, for the SMDP model, we are considering an additional dependency on $h \in [0,H]$. \\
Notation used:
\begin{itemize}
    \item $H$ is the horizon
    \item $\mu$ policy over options $\{\mu: S \times O \times H \rightarrow [0,1]\}$
    \item $r(s,o,h)$ is the discounted cumulative reward gained by selecting the option $o$, in state $s$, in the instant $h$ of the horizon $H$
    \item $p(s',h'|s,o,h)$ is a new transition model that characterizes both the state dynamic and the time the option executes.
    \item $w(s,o,h)$ is the probability of playing option $o$ being in state $s$ at time-step $h$
\end{itemize}

The value function is defined as:
\begin{align}
\label{eq:value_smdp}
    V^{\mu}(s, h) 
    &= \sum_{o \in O_s} \mu(s,o,h) \Bigl[r(s,o,h) + \sum_{s', h'>h}p(h',s'|s,o,h) V^{\mu}(s', h') \Bigr]
\end{align}
with $V^{\mu}(s,H)=0$.

\subsection{Performance Difference Lemma}
Difference in value of a policy $\mu$ in two different SMDPs ( $\tilde{}$ and $\bar{}$ )
\begin{gather*}
    \tilde{V}^{\mu_k}(s,1)-\bar{V}^{\mu_k}(s,1) \\ 
    \resizebox{.96\textwidth}{!}{$\displaystyle
    = \hat{\mathbb{E}}\bigg[ \sum_{i=1}^H \Big(\big(\tilde{r}(s_i,o_i,h_i)-\bar{r}(s_i,o_i,h_i)\big) + \big(\tilde{p}(s_{i+1}, h_{i+1}|s_i,o_i,h_i)-\bar{p}(s_{i+1}, h_{i+1}|s_i,o_i,h_i)\big)\tilde{V}^{\mu_k}(s_{i+1},h_{i+1})\Big) \mathbbm{1}\{h_i < H\} \bigg]$}
\end{gather*}
$\hat{\mathbb{E}}$ is the expectation taken w.r.t. the policy $\mu$ and the transition probability $\hat{p}(s',h'|s,o,h)$, and can be rewrite as:
\begin{align*}
    \prod_{i=1}^H \mu_k(s_i,o_i,h_i)\hat{p}_k(s_{i+1},h_{i+1}|s_i,o_i,h_i)\mathbbm{1}\{h_i<H\}
\end{align*}
This quantity is the distribution of visits for the policy $\mu_k$ in the " $\hat{}$ " SMDP and it is equivalent to $w_{tk}$ for the FH-MDP case.
\begin{proof}
    The result follows by unrolling equation \ref{eq:value_smdp}. Lemma E.5 \citet{dann2017unifying} for an example in FH-MDPs.
\end{proof}

\subsection{Confidence Intervals}
The confidence sets are defined as:
\begin{align*}
    &B^r_{k} (s,o,h) := [\hat{r}_{k}(s,o,h) - \beta^r_{k}(s,o,h),\ \hat{r}_{k}(s,o,h) + \beta^r_{k}(s,o,h)] \\
    &B^p_{k} (s,o,h) := \bigl\{ p_{k}(\cdot,\cdot|s,o,h) \in \nabla(s):  \| \tilde{p}_{k}(\cdot,\cdot|s,o,h) - \hat{p}_{k}(\cdot,\cdot|s,o,h) \|_1 \leq \beta^p_{k}(s,o,h) \bigr\} 
\end{align*}
and the relative confidence bounds $\beta^r_k(s,o,h)$ and $\beta^p_k(s,o,h)$ using Empirical Bernstein bound \citep{maurer2009empirical}, \cite{hoeffding1963probability} and \cite{weissman2003inequalities}.
\begin{align}
    &\beta^r_{k}(s,o,h)\ \propto\ \sqrt{\frac{2 \hat{\Var}(r) \ln 2/\delta}{n(s,o,h)}} + \frac{7\ln2/\delta}{3(n-1)}\\
    &\beta^p_{k}(s,o,h)\ \propto\ \sqrt{\frac{S\log\bigl(\frac{n_{k}(s,o,h)}{\delta}\bigl)}{n_{k}(s,o,h)}} 
    \label{eq:ci_option} 
\end{align}
with $\Hat{\Var}(r)$ be the sample variance of r.
\begin{align}
    \Hat{\Var}(r) = \frac{1}{n(n-1)}\sum_{1 \leq i \leq j \leq n} (r_i - r_j)^2
\end{align}

\subsection{Actual Proof}
\regretsmdp*
\begin{proof}
\begin{align*}
    \textit{Regret(K)} &= \sum_{k=1}^K V^*(s,1) - \bar{V}^{\mu_k}(s,1) \\
    &\myleq{Opt.} \sum_{k=1}^K \tilde{V}^{\mu_k}(s,1) - \bar{V}^{\mu_k}(s,1)\\
    &=\sum_{k=1}^K \bar{\mathbb{E}}\bigg[ \sum_{i=1}^H \Big(\big(\tilde{r}(s_i,o_i,h_i)-\bar{r}(s_i,o_i,h_i)\big) + \big(\tilde{p}(s_{i+1}, h_{i+1}|s_i,o_i,h_i)-\bar{p}(s_{i+1}, h_{i+1}|s_i,o_i,h_i)\big)\tilde{V}^{\mu_k}(s_{i+1},h_{i+1})\Big) \\
    &\mathbbm{1}\{h_i < H\} \bigg] \\
    &\myeq{a} \sum_k \sum_{i \in [H]} \sum_{(s,o,h) \in L_k} w_{k}(s_i,o_i,h_i) \biggl( \bigl(\tilde{r}(s_i,o_i,h_i) - \bar{r}(s_i,o_i,h_i) \bigr) \\
    &+\bigl( \tilde{p}(s_{i+1}, h_{i+1}|s_i, o_i, h_i) - \bar{p}(s_{i+1}, h_{i+1}|s_i, o_i, h_i)\bigr)^T  \tilde{V}^{\mu_k}(s_{i+1}, h_{i+1}) \biggr) \\
    &\myleq{b} \sum_k \sum_{i \in [H]} \sum_{(s,o,h) \in L_k} w_{k}(s_i,o_i,h_i) \Big( 2\beta^r_k(s_i,o_i,h_i) + 2\beta^p_k(s_i,o_i,h_i)^T H\Big) \\
    &\mypropto{c} \sum_k \sum_{i \in [H]} \sum_{(s,o,h) \in L_k} w_{k}(s_i,o_i,h_i) \bigg(\sqrt{\frac{\Hat{\Var}(r)}{n_k(s,o,h)}} + \frac{1}{n_k(s,o,h)-1} + \sqrt{\frac{S}{n_k(s,o,h)}} H\bigg) \\    
\end{align*}
\begin{align*}
    &\myleq{d} \sum_k \sum_{i \in [H]} \sum_{(s,o,h) \in L_k} \frac{w_{k}(s_i,o_i,h_i)}{\sqrt{n_k(s_i,o_i,h_i)}} \bigg(\sqrt{\Hat{\Var}(r)} + \sqrt{S} H\bigg) + \sum_k \sum_{i \in [H]} \sum_{(s,o,h) \in L_k} \frac{w_k(s_i,o_i,h_i)}{n_k(s,o,h)}\\
    &\myleq{e} \tilde{O}\Bigg(\Big(\sqrt{dSOn}\Big) \bigg(\sqrt{\Hat{\Var}(r)} + \sqrt{S} H\bigg) + dSO\Bigg)\\
    &\myleq{f} \tilde{O}\Bigg(\Big(\sqrt{dSOn}\Big) \bigg(R_{max} \overline{T} + \sqrt{S} H\bigg)  +  dSO\Bigg)
\end{align*}
with 
\begin{align}
    \overline{T} = \max_{s,o,h} \sqrt{\E[\tau(o_i, s_i, h_i)^2]} = \max_{s,o,h} \sqrt{\E[\tau(o_i, s_i, h_i)]^2 + \text{Var}[\tau(s_i,o_i,h_i)]}
\end{align}
with $\tau$ representing the average duration of the set of options seen so far.

The first passage is a standard inequality when proving the regret in frameworks adopting optimism in face of uncertainty.
\begin{enumerate}[label=(\alph*)]
    \item The expectation with respect to the policy $\mu_k$ and the transition model $\bar{p}$ can be replaced with a more common formulation used in the Finite Horizon literature \citep{dann2017unifying, zanette2018}, $\sum_{(s,o,h) \in L_k} $.\\
    Where, $L_k$ is defined as the good set \citep{dann2017unifying, zanette2018}, which is the number of episodes in which the triple $(s,o,h)$ is seen sufficiently often, and this equation is valid for all the tuples $(s, o, h)$ being part of this set.
    \item We upper bound the difference of rewards and transition probabilities with two times their relative confidence intervals, and, the Value function at the next step with the horizon length $H$.
    \item We substitute the confidence intervals with their definitions (eq. \ref{eq:ci_option}) neglecting logarithmic terms.
    \item We divide the summation in two, to upper bound the terms separately
    \item Using the adaptation of lemma 16 of \cite{zanette2018} for SMDPs, lemma \ref{lemma:zan_smdp}, for the first term. Using passage (b) and (c) in the proof of lemma \ref{lemma:zan16}
    \item Upperbounding the sample variance of $r$, with $R_{max} \overline{T}$. Where $\overline{T}$ is the sample variance of the duration.
\end{enumerate}
\end{proof}

\section{Special Case of fixed-length options}
Let's consider the same finite horizon MDP with options with fixed length $M:= <S, O, R_h, P_h, H, \bar{\tau}>$ where each option $o:= <I, \pi^o, \beta>$ has a fixed initial set $I$ and fixed termination condition $\beta$, $R_h(s,o) \in [0, \bar{\tau} R_{max}]$ is the expectation of the reward function distribution, $P_h(\cdot |s,o)$ is the transition distribution, $H$ is the horizon, and $\btau \leq H$ is the options fixed length. The solution of the MDP will be a policy $\pi^H: S \rightarrow O$ that maximizes the cumulative return choosing among options' optimal policies $\pi^{o_i}$.
The reward function over $state-options$ pairs relates to the flat-MDP's reward as:
\begin{align*}
    R_h(s,o) = \E_{\substack{s_0 = s \\ a_i \sim \pi^o(\cdot | s_i) \\ s_{i+1} \sim p_{h+1} (\cdot | s_i, a_i)}} \biggl[\sum_{i=0}^{\btau-1} r_{h+i}(a_i, s_i)\biggr]
\end{align*}
Denote with $V^{\pi^H}_n (s)$ the state value function associated with a hierarchical policy $\pi^H$ (with Hierarchical policy we define a policy that chooses among options).
\begin{align*}
    V^{\pi^H}_n (s) = \E_{\substack{s_0 = s \\ o_j \sim \pi^H(\cdot | s_j) \\ s_{j+1} \sim P(\cdot|s_j, o_j)_h}} \biggl[ \sum_{j=0}^N R_{h \times\btau}(s_j, o_j) \biggr]
\end{align*}
with $N = \frac{H}{\btau}$ the number of decision steps that occurs during the Horizon. \\
In this way, we can exploit the same \textit{performance difference lemma} of Dann \citep{dann2017unifying} Lemma E.15, where, instead of actions we have fixed length options, and we sum over $N$ decision steps. Hence, we can write:

\begin{align}
    \textit{Regret(K)} &\myleq{Opt.} \sum_{k=1}^K \tilde{V}^{\pi^H_k}_1(s) - \bar{V}^{\pi^H_k}_1(s) \\
    &\myeq{a} \sum_k \sum_{n \in [N]} \sum_{(s,o) \in L_k} w_{nk}(s,o) \biggl( \bigl(\tilde{R}_n(s,o) - \bar{R}_n(s,o) \bigr) +\bigl( \tilde{P}_n(s,o) - \bar{P}_n(s,o)\bigr)^T  \tilde{V}^{\pi_k}_{n+1} \biggr) \label{eq:gs} \\ 
    &+ \textit{term considering the state-options pairs inside the failure event}
\end{align}

here $w_{nk}(s,o)$ is the probability of visiting state $s$ and choosing option $o$ there at the decision step n in the $k$-th episode.

Then, we consider a new formulation of the confidence sets:
\begin{align*}
    &B^R_{nk} (s,o) := [\hat{R}_{nk}(s,o) - \beta^R_{nk}(s,o),\ \hat{R}_{nk}(s,o) + \beta^R_{nk}(s,o)] \\
    &B^P_{nk} (s,o) := \bigl\{ P_{nk}(\cdot|s, o) \in \nabla(s):  \| \tilde{P}_{nk}(\cdot|s,o) - \hat{P}_{nk}(\cdot|s,o) \|_1 \leq \beta^P_{nk}(s,o) \bigr\} 
\end{align*}

using \cite{hoeffding1963probability} and \cite{weissman2003inequalities} the confidence bounds are: 

\begin{align}
    &\beta^R_{nk}(s,o)\ \propto\ R_{max}\btau \sqrt{\frac{\log\bigl(\frac{n_{nk}(s,o)}{\delta}\bigr)}{n_{nk}(s,o)}} \\
    &\beta^P_{nk}(s,o)\ \propto\ \sqrt{\frac{S\log\bigl(\frac{n_{nk}(s,o)}{\delta}\bigl)}{n_{nk}(s,o)}}
\end{align}
 
After the definition of the confidence sets we can bound the previous equation as follow: 
\begin{align*}
    \sum_{k=1}^K \tilde{V}^{\pi^H_k}_1(s) - \bar{V}^{\pi^H_k}_1(s)
    &= \sum_k \sum_{n \in [N]} \sum_{(s,o) \in L_k} w_{nk}(s,o) \biggl( \bigl(\tilde{R}_{nk}(s,o) - \bar{R}_{nk}(s,o) \bigr) +\bigl( \tilde{P}_{nk}(s,o) - \bar{P}_{nk}(s,o) \bigr)^T  \tilde{V}^{\pi_k}_{h+1} \biggr) \\
    &+ \textit{term considering the state-options pairs inside the failure event} \\
    &\myleq{a} \sum_k \sum_{n \in [N]} \sum_{(s,o) \in L_k} w_{nk}(s,o) \biggl( 2 \beta^R_{nk} + 2{\beta^P_{nk}}^T  (H - \btau)\biggr) \\
    &\mypropto{b} \sum_k \sum_{n \in [N]} \sum_{(s,o) \in L_k} w_{nk}(s,o) \biggl( \frac{R_{max}\btau}{\sqrt{n_{nk}}} + \sqrt{\frac{S}{n_{nk}}}  (H - \btau)\biggr) \\
    &= \sum_k \sum_{n \in [N]} \sum_{(s,o) \in L_k} \frac{w_{nk}(s,o)}{\sqrt{n_{nk}}} \biggl( R_{max}\btau+ \sqrt{S}  (H - \btau)\biggr) \\
    &\myleq{c}\tilde{O}\Bigl(N\sqrt{SOK} \bigl(R_{max} \btau + \sqrt{S}H - \sqrt{S}\btau\bigr)\Bigr) \\
    &\myleq{d}\tilde{O}\Bigl(N\sqrt{SOK} \bigl(R_{max} \btau + \sqrt{S}H)\Bigr)
\end{align*}
\begin{enumerate}[label=(\alph*)]
    \item substituting the $\tilde{R}_n(s,o) - \bar{R}_n(s,o)$ and $\tilde{P}_n(s,o) - \bar{P}_n(s,o)$ with double the relative confidence interval and considering $\tilde{V}^{\pi_k}_{h+1} \leq (H - \btau)$. The second term will be omitted for ease of notation.
    \item replacing the confidence intervals with their definition
    \item Lemma \ref{lemma:zan_optfl}
    \item considering the worst case, where there isn't the negative term
\end{enumerate}
comparing it with the bound of \cite{fruit2017exploration} for bounded holding time:
\begin{align*}
    \tilde{O} \Bigl( \bigl(D_o \sqrt{S} + T_{max} + (T_{max} - T_{min})\bigr) R_{max} \sqrt{S_O O n } \Bigr)
\end{align*}
that having options with fixed duration $\btau$, and considering $R_{max} = 1$ reduces to:
\begin{align*}
    \tilde{O} \Bigl( D S \sqrt{O n} + \btau\sqrt{S O n } \Bigr)
\end{align*}
we have the same bound where instead of the diameter we have the Horizon $H$, and where $NK$ is exactly equal to the number of the decisions up to episode $k$, which is $n$ in their notation. We have:
\begin{align*}
    \tilde{O} \Bigl( H S \sqrt{O N^2 K} + \btau\sqrt{S O N^2 K} \Bigr)
\end{align*}
\textbf{Important}: Note that we have an additional $\sqrt{N}$ terms because we considered non-stationary MDP. This is a well-known penalty term when considering non-stationarity in the process.

\section{Proof of Theorem \ref{thm:two-phase}}
\twophase*

\begin{proof}
The regret of the two-phase algorithm can be written in this form
\begin{align*}
    Regret(K) &= \sum_{k=1}^{K_1} V_*^*(s,1) - V_{(\pi_k)}^{\mu}(s,1) + \sum_{k=k_1} V_*^*(s,1) - V^{\mu_k}_{\pi_{K_1}}\\
    &= \underbrace{\sum_{k=1}^{K_1} V_*^*(s,1) - V_{(\pi_k)}^{\mu}(s,1)}_{\text{Options learning Regret}} + \sum_{k=K_1}^K \underbrace{V_*^*(s,1) - V^*_{(\pi_{K_1})}}_{\text{Bias}} + \underbrace{V^*_{(\pi_{K_1})} - V^{\mu_k}_{(\pi_{K_1})}}_{\text{Regret SMDP with fixed options}}
\end{align*}
The regret is the sum of the regret paid in the first phase and the regret paid in the second one, plus an additional bias term.
By assuming all options with equal samples $K_o$ and the options' policies learning as $O$ finite horizon MDPs for which $S_o, A_o, H_o$ are the upper bounds of the option's state space dimension, the option's action space dimension and the option's horizon, and $K_1 = \sum_{o \in O} K_o$ and 
\begin{align}
\label{eq:k}
    K = \sum_{o \in O} K_o + K_2
\end{align}
we can write the regret as
\begin{align*}
    Regret(K) \leq \sum_{o \in O}K_oH_o + K_2 \max_{o \in O} \frac{1}{K_o}H^2_oS_o\sqrt{A_oK_o}+ HS\sqrt{Od^2K_2}
\end{align*}
In which we pay full regret for each option learning, then the maximum average regret considering the option learning as a finite horizon MDP with horizon $H_o$, and the regret of the SMDP learning with fixed options. \\
However, considering the options with equal samples and with $S_o, A_o, H_o$ the upper bounds of the relative quantities we can get rid of the maximization in the second term, and $d=\frac{H}{H_o}$ and the regret became
\begin{align*}
        Regret(K) &\leq OK_oH_o + \frac{K_2}{K_o}H^2_oS_o\sqrt{A_oK_o}+ \frac{H^2}{H_o}S\sqrt{OK_2}
\end{align*}
Now, by substituting $K_2$ with eq. \ref{eq:k}, and upper bounding $(K-OK) \leq K$ we can solve in closed form to find $K_o$ to minimize the regret.
\begin{align*}
        &Regret(K) \leq OK_oH_o + \frac{K}{K_o}H^2_oS_o\sqrt{A_oK_o}+ \frac{H^2}{H_o}S\sqrt{OK} \\
        &K_o = \sqrt[3]{\frac{K^2S^2_oH^2_oA_o}{O^24}}
\end{align*}

Therefore, by substituting $K_o$ in the original equation we have
\begin{align*}
    Regret(K) \leq \Tilde{O}\Big(K^{\frac{2}{3}}(H^5_oS^2_oA_oO)^{\frac{1}{3}} + \frac{H^2S}{H_o}\sqrt{OK}\Big)
\end{align*}
\end{proof}

Now we can compare the regret of this algorithm compared to the regret of UCRL2 adapted for non-stationary FH-MDPs \citep{glp2020aaaitutorial}.
\begin{align*}
    Regret(UCRL2-CH) \leq \tilde{O}(H^2S\sqrt{AK})
\end{align*}

\begin{align}
    &\frac{Regret_{SMDP}}{Regret_{MDP}} \leq \frac{K^{1/6}\alpha^{3/8}O^{1/3}}{(HS)^{1/3}A^{1/6}} \leq 1\\
    &K \leq \frac{H^2S^2A}{\alpha^{16} O^2}
\end{align}

\section{Renewal Processes}

\begin{lemma}[Renewal Function Bound]
\label{lemma:ren-fun}
    Considering a Renewal process, $(X_t)_{t \geq 0}$, and a sequence $S_1, S_2 \dots $ of random variables, characterizing the random duration of an event, alternatively defined as holding time, with $supp(S_i) \in \{1, \dots, H\}$. We can bound, with probability $1-\delta$, the expected number of random events that occurred up to time $t$, $X_t$, with:
    \begin{align*}
        X_{t} < \sqrt{\frac{\ln2 - \ln{\delta}}{cK}} + \frac{t}{\mu}
    \end{align*}
    with $c = \frac{\mu^3}{32\sigma^2T}$ where $\mu$ is the mean of the r.v.s and $\sigma^2$ the variance.
\end{lemma}

\begin{proof}
    Based on the proof presented on \cite{renewal2019}, which apply DKW type inequalities to renewal processes \citep{dvoretzky1956asymptotic}
    \begin{align*}
        \Pr \left(\sup_{0 \leq t \leq T} \left|{\frac{X_{nt}}{n} - \frac{t}{\mu}} \right| \geq \epsilon \right) \leq 2e^{-cn\epsilon^2}
    \end{align*}
    Now we can equal $2e^{-cn\epsilon^2}$ to $\delta$ and find $\epsilon$.
    \begin{align*}
        \epsilon = \sqrt{\frac{\ln 2 - \ln \delta}{cn}}  
    \end{align*}
    Thus with probability $1- \delta$
    \begin{align*}
        X_t \leq \sqrt{\frac{\ln2 - \ln{\delta}}{cn}} + \frac{t}{\mu}
    \end{align*}
    that completes the proof
\end{proof}

The following lemma is lemma \ref{lemma:opt-count} in the main paper.
\optioncount*
\begin{proof}
    The proof followed the one of lemma \ref{lemma:ren-fun} and the fact that we are considering $T = H$, $n = 1$, $t=H$, $\mu = \Bar{\tau}$, $\sigma^2 = (\tau_{max}-\tau_{min})$, and $X_t=d$.
\end{proof}

\section{Usefull Lemmas}
\begin{lemma}[lemma 16 \citep{zanette2018} for non stationary MDPs]\label{lemma:zan16}
The following holds true:
    \begin{align*}
        \sum_k \sum_{h \in [H]} \sum_{(s,a) \in L_k} w_{hk}(s,a) \sqrt{\frac{1}{n_k(s,a,h)}} = \Tilde{O}(\sqrt{HSAT})
    \end{align*}
where the extra $\sqrt{H}$ is due to the non-stationarity of the environment
\end{lemma}

\begin{proof}
    \begin{align*}
        & \sum_k \sum_{h \in [H]} \sum_{(s,a) \in L_k} w_{hk}(s,a) \sqrt{\frac{1}{n_k(s,a,h)}} \\
        &\myleq{a} \sqrt{\sum_k \sum_{h \in [H]} \sum_{(s,a) \in L_k} w_{hk}(s,a)} \sqrt{\sum_k \sum_{h \in [H]} \sum_{(s,a) \in L_k} w_{hk}(s,a) \frac{1}{n_k(s,a,h)}} \\
        &\myeq{b}\sqrt{KH} \sqrt{\sum_k \sum_{h \in [H]} \sum_{(s,a) \in L_k} w_{hk}(s,a) \frac{1}{n_k(s,a,h)}} \\
        &\myleq{e}\tilde{O}(\sqrt{HSAT})
    \end{align*}
    Then: 
    \begin{align*}
        &\sum_k \sum_{h \in [H]} \sum_{(s,a) \in L_k} \frac{w_{hk}(s,a)}{n_k(s,a,h)} \\
        &\myleq{c} \sum_{h \in [H]} \sum_{(s,a) \in L_k} \sum_k \frac{w_{hk}(s,a)}{\frac{1}{4} \sum_{j \leq k} w_{hj}(s,a)} \\
        &\myleq{d} 4 HSA \log \biggl(\frac{Ke}{w_{min}} \biggr)\\ 
        &\mypropto{$\sim$} HSA
    \end{align*}
    
    \begin{enumerate}[label=(\alph*)]
        \item by Cauchy-Schwartz
        \item $\sum_{t \in [H]} \sum_{(s,a) \in L_k} w_{tk}(s,a) = H$ lemma 17 (\textit{b}) \cite{zanette2018}  
        \item lemma 2 \cite{zanette2018} adapted to the non-stationary case
        \item lemma E.5 \cite{dann2017unifying} considering that being (s,a) part of the good set $L_k$, then we are assuming (Appendix E.3 \cite{dann2017unifying}) that $w_k(s,a) \geq w_{min}$. 
        \item substituting (f) we get the upper bound, and we conclude the proof.
    \end{enumerate}
\end{proof}

\begin{lemma}[lemma 16 \citep{zanette2018} for SMDPs (Lemma \ref{lemma:zan_smdp} main paper)] 
Considering a non-stationary MDP M with a set of options as an SMDP $M_{\mathcal{O}}$ \citep{sutton1999between}. In $M_{\mathcal{O}}$ the number of decisions taken in the $k^{th}$-episode is a random variable $d$ and
\begin{align*}
        \sum_{i \in H} \sum_{(s,o) \in L_k} w_k(s_i,o_i,h_i) \mathbbm{1}\{h_i < H\} = d_k \ \text{with} \ \{\forall k: d_k \leq H\}  
\end{align*}
with mean $d$.
Therefore, the following holds true:
    \begin{align*}
        \sum_k \sum_{i \in H} \sum_{(s,o) \in L_k} w_k(s_i,o_i,h_i) \sqrt{\frac{1}{n_k(s_i,o_i,h_i)}} = \tilde{O} \biggl(\sqrt{SOKd^2} \biggr)
    \end{align*}
or, more generally, using the same notation used in \cite{fruit2017exploration}
    \begin{align*}
        \sum_k \sum_{i \in H} \sum_{(s,o) \in L_k} w_k(s_i,o_i,h_i) \sqrt{\frac{1}{n_k(s_i,o_i,h_i)}} = \tilde{O} \biggl(\sqrt{dSOn} \biggr)
    \end{align*}
with $n$ number of decisions taken up to episode $k$.
\end{lemma}

\begin{proof}
Due to the stochasticity of the option's duration, $d$ is a random variable expressing the number of decisions taken in a step. Thus, first, we can rewrite passage $(b)$ of the proof of lemma 17 \cite{zanette2018} then,  we change lemma \ref{lemma:zan16} considering the same notion of good set considered in the appendix of \cite{zanette2018} and the validity of lemma 2 of \cite{zanette2018}, in the options framework(replacing $o$ with $a$).
If all the aforementioned assumptions hold, thus the derivation of the new lemma follows the derivation of lemma \ref{lemma:zan16}
\end{proof}

\begin{lemma}[lemma 16 \citep{zanette2018} for MDPs with options of fixed lenght] \label{lemma:zan_optfl}
For an MDP with $O$ options, with a fixed lenght $\Bar{\tau}$, where the horizon is divided in $N = \frac{H}{\Bar{\tau}}$ decision steps, the following holds true:
    \begin{align*}
        \sum_k \sum_{n \in N} \sum_{(s,o) \in L_k} w_{nk}(s,o) \sqrt{\frac{1}{n_k(s,o)}} = \tilde{O} \biggl(N\sqrt{SOK} \biggr)
    \end{align*}
\end{lemma}

\begin{proof}
    In this MDP the control returns to the hierarchical policy after exactly $\Bar{\tau}$ time steps (the length of an option), thus, we can have at most $N = \frac{H}{\Bar{\tau}}$ actions in the horizon $H$. For this reason, passage (b) of the proof of lemma \ref{lemma:zan16} become 
    \begin{align*}
        \sum_{n \in N} \sum_{(s,o) \in L_k} w_{nk}(s,o) = N
    \end{align*}
    The rest results for the same passage of the proof of lemma \ref{lemma:zan16}.
\end{proof}

To have a more complete analysis we need also to consider the triples (s, o, h) which aren't inside the good set. To do that, we can adapt Lemma 3 of \citet{zanette2018}, for the FH-SMDP setting.
\begin{lemma}[Outside the good set]
It holds that:
\begin{align*}
    \sum_{k=1}^K \sum_{h=1}^{d} \sum_{(s,o,h) \notin L_k} w_{k}(s,o,h) = \Tilde{O}(SOd)
\end{align*}
\end{lemma}
The proof follows from the one of lemma 3 of \citet{zanette2018}.

\end{document}